\documentclass{article}

\usepackage{microtype}
\usepackage{graphicx}
\usepackage{subfigure}
\usepackage{booktabs} 
\usepackage[utf8]{inputenc}
\usepackage[T1]{fontenc}
\usepackage{parskip}
\usepackage{amsmath,amssymb,amsthm}
\usepackage{float}
\usepackage{xcolor}
\usepackage{comment}
\usepackage{mathrsfs}
\usepackage{chngcntr}
\usepackage{apptools}
\AtAppendix{\counterwithin{lemma}{section}}
\AtAppendix{\counterwithin{proposition}{section}}
\AtAppendix{\counterwithin{theorem}{section}}

\usepackage{hyperref}

\usepackage[accepted]{icml2021}

\newcommand{\Op}{\mathrm{Op}}

\newcommand{\OT}{\mathrm{OT}}
\newcommand{\Expect}{\mathbb{E}}
\newcommand{\Var}{\mathrm{Var}}
\newcommand{\ADM}{\mathrm{ADM}}
\newcommand{\Tr}{\mathrm{Tr}}
\newcommand{\E}{\mathbb{E}}
\newcommand{\Probs}{\mathcal{P}}
\newcommand{\N}{\mathcal{N}}
\newcommand{\M}{\mathcal{M}}
\newcommand{\GP}{\mathcal{GP}}
\newcommand{\KL}{D_{\mathrm{KL}}}

\renewcommand{\H}{\mathcal{H}}
\newcommand{\Sym}{\mathrm{Sym}}
\newcommand{\HS}{\mathrm{HS}}

\newcommand{\Ncal}{\mathcal{N}}

\newcommand{\X}{\mathcal{X}}
\newcommand{\R}{\mathbb{R}}

\newcommand{\defeq}{:=}

\newtheorem{remark}{Remark}
\newtheorem{theorem}{Theorem}
\newtheorem{proposition}{Proposition}
\newtheorem{lemma}{Lemma}

\icmltitlerunning{Estimating 2-Sinkhorn Divergence between GPs}

\begin{document}

\twocolumn[
\icmltitle{Estimating 2-Sinkhorn Divergence between Gaussian Processes from Finite-Dimensional Marginals}



\icmlsetsymbol{equal}{*}

\begin{icmlauthorlist}
\icmlauthor{Anton Mallasto}{aalto}
\end{icmlauthorlist}

\icmlaffiliation{aalto}{Department of Computer Science, Aalto University, Finland}

\icmlcorrespondingauthor{Anton Mallasto}{anton.mallasto@aalto.fi}

\icmlkeywords{Machine Learning, ICML}

\vskip 0.3in
]



\printAffiliationsAndNotice{\icmlEqualContribution} 

\begin{abstract}
\emph{Optimal Transport} (OT) has emerged as an important computational tool in machine learning and computer vision, providing a geometrical framework for studying probability measures. OT unfortunately suffers from the curse of dimensionality and requires regularization for practical computations, of which the \emph{entropic regularization} is a popular choice, which can be 'unbiased', resulting in a \emph{Sinkhorn divergence}. In this work, we study the convergence of estimating the 2-Sinkhorn divergence between \emph{Gaussian processes} (GPs) using their finite-dimensional marginal distributions. We show almost sure convergence of the divergence when the marginals are sampled according to some base measure. Furthermore, we show that using $n$ marginals the estimation error of the divergence scales in a dimension-free way as $\mathcal{O}\left(\epsilon^ {-1}n^{-\frac{1}{2}}\right)$, where $\epsilon$ is the magnitude of entropic regularization.
\end{abstract}

\section{Introduction}

\emph{Gaussian processes} (GPs) are infinite-dimensional counterparts of normal distributions, which are plentiful in machine learning tasks where one aims to infer functional relationships with uncertainty estimates \citep{hauberg2015random,le2015gpssi,roberts2013gaussian}. Centered GPs correspond to \emph{covariance operators}, which appear as features in computer vision~\citep{faraki2015approximate,harandi2014bregman} and natural language processing \citep{pigoli2014distances}. To compare such operators with each other requires defining a metric or a divergence. Such divergences exist readily for the finite-dimensional covariance matrices~\cite{arsigny2006log,pennec2006riemannian}, but these expressions do not always extend in a straight-forward manner to the infinite-dimensional setting, mainly due to the eigenvalues of covariance operators converging to zero~\cite{haquang2014log,haquang2015affine}. As the covariance operators correspond to centered probability measures, \emph{optimal transport} (OT) can be applied to compute well-defined divergences.

OT~\cite{villani08,peyre2019computational} provides a geometrical toolkit for studying the space of probability measures by extending a cost function between samples of the probabilities to a divergence between the entire probability measures. A special subclass of these divergences is formed by the \emph{Wasserstein metrics}, which are proper metrics, and therefore enjoy favorable topological properties compared to information theoretical divergences, such as the \emph{Kullback-Leibler (KL) divergence}. However, the OT problem requires regularization due to bad scaling properties with respect to the sample dimension, behaving as $\mathcal{O}\left(n^{-\frac{1}{d}}\right)$~\cite{dudley1969speed}. To combat this, a popular approach is to combine OT and information theoretical divergences, by regularizing the OT problem with a divergence term~\cite{dessein18,dimarino20}, e.g. the KL-divergence, resulting in \emph{entropy-regularized OT}~\cite{cuturi13}, improving the sample complexity to $\mathcal{O}\left((1+\epsilon^{-\lfloor d/2\rfloor})n^{-\frac{1}{2}}\right)$~\cite{genevay18sample}, which has been further sharpened to $\mathcal{O}\left(\epsilon\left(1+\frac{\sigma^{\lceil 5d/2 \rceil+6}}{\epsilon^{\lceil 5d/4 \rceil+3}}\right)n^{-\frac{1}{2}}\right)$~\citep{mena2019statistical}.

OT approaches to compute divergences between Gaussians have been studied extensively, as closed form expressions can be derived. The $2$-Wasserstein distance between normal distributions was studied in~\citep{givens84,dowson82,olkin82,knott84}, $2$-Wasserstein distance between GPs in~\citep{gelbrich90,mallasto17,masarotto2019procrustes}. Also the entropy-regularized $2$-Wasserstein distance and the Sinkhorn divergence between normal distributions have garnered considerable attention lately ~\citep{Mallasto2020entropyregularized,Janati2020entropicOT,GerGroGor19,del2020statistical,RipThesis}, and their extensions to Gaussian measures and GPs~\citep{haquang20entropic}.

In this work, we study the complexity of computing the entropy-regularized $2$-Wasserstein distance as well as the $2$-Sinkhorn divergence between two GPs using a finite amount of their marginals. In other words, the question is: given two GPs $f_0,f_1$ over an index set $\X$, over how many samples $\{x_1,...,x_n\}\subseteq \X$ should we evaluate the GPs to get a reasonable approximation for the $2$-Sinkhorn divergence with regularization $\epsilon$ between the GPs?

The contributions can be summarized as follows:
\begin{itemize}
    \item We show that the entropy-regularized 2-Wasserstein distance has a marginal complexity (error rate as a function of marginals) of $\mathcal{O}\left(\frac{1}{\sqrt{n}}\left(\frac{1}{\epsilon}+\mathrm{const.}\right)\right)$ and the $2$-Sinkhorn divergence $\mathcal{O}\left(\frac{1}{\epsilon \sqrt{n}}\right)$.
    \item We provide concentration bounds for the errors of the empirical estimates for the two divergences.
    \item We illustrate empirically the convergence of the estimates, and observe that increasing $\epsilon$ does not necessarily increase estimation accuracy in terms of relative error as we increase $n$. However, if we increase the dimension of $\X$, larger $\epsilon$ also decreases relative error.
\end{itemize}

\section{Background}
In the following, we briefly summarize some prerequisites and fix notation.
\subsection{Gaussian Processes and Covariance Operators}

\paragraph{Gaussian processes.} A \emph{Gaussian process} (GP) $f$ is a collection of random variables, such that any finite restriction $(f(x_i))_{i=1}^n$ has a joint Gaussian distribution, where $x_i\in \X$, and $(\X,\rho)$ is the \emph{index set} equipped with a probability measure $\rho$. We will assume $\X\subset \R^d$ to be compact. A GP is entirely characterized by the pair
\begin{equation}
\begin{aligned}
m(x)&=\E{f(x)},\\
K(x,x')&=\E\left(f(x)-m(x))(f(x')-m(x')\right)~,
\end{aligned}
\end{equation}
where $m$ and $K$ are called the \emph{mean function} and the \emph{covariance function} (or \emph{kernel}), respectively, and denote such a GP by $f\sim \GP(m,K)$. It follows from the definition that the covariance function $K$ is symmetric and positive semidefinite \emph{kernel}. The marginal of $f$ over $X=(x_1,...,x_n)\subseteq \X$ follows the Gaussian distribution $f(X) \sim \N(\mu, \Sigma)$, where $\mu_i = m(x_i)$ and $\Sigma_{ij} = K(x_i,x_j)$.

\paragraph{Covariance operators.} Denote by $\H = L^2(\X,\rho)$ the space of $L^2$-integrable functions from $\X$ to $\mathbb{R}$ under the reference probability measure $\rho$. The kernel $K(\cdot, \cdot)$ has an associated integral operator, denoted by abuse of notation as $K \colon L^2(\X) \to L^2(\X)$, defined by
\begin{equation}
[K\phi](x)=\int_\X K(x,s)\phi(s)d\rho(s), ~\forall \phi\in \H~,
\end{equation}
called the \emph{covariance operator}. The operator $K$ is \emph{self-adjoint} and \emph{positive}, which we denote by $K\in\Sym^+(\H)$, and it is of \emph{trace-class}, denoted by $K\in \Tr(\H)$. That is, the \emph{trace}
\begin{equation}
\Tr \left[K\right] = \sum_{k=1}^\infty \langle K e_k,e_k\rangle_\H,
\end{equation}
is finite, absolutely convergent and independent of the basis $\{e_k\}_{k=1}^\infty$ of $\H$. If the kernel $K$ is continuous, then
\begin{equation}
    \Tr \left[K\right] = \int_\X K(x,x) d\rho(x).
\end{equation}
Associated with the trace is the \emph{trace norm} of an operator
\begin{equation}
    \|K\|_\Tr \defeq \Tr\left[|K|\right]<\infty,
\end{equation}
where $|K|=(K^*K)^\frac{1}{2}$ is the \emph{absolute value} of $K$ and $K^\frac{1}{2}$ is the \emph{operator square-root} of $K$. If $K$ is positive and self-adjoint, then $K=|K|$. The trace norm belongs to the family of \emph{Schatten $p$-norms}, given by
\begin{equation}
    \|A\|_p \defeq \left(\sum_{k=1}^\infty \sigma_k(A)^p\right)^\frac{1}{p} = \left(\Tr \left[|A|^p\right]\right)^\frac{1}{p},
\end{equation}
where $\sigma_i(A)$ denotes the i\textsuperscript{th} largest eigenvalue of $A$. The Schatten $p$-norms are \emph{unitarily invariant}~\cite{Bhatia:1997Matrix}. The \emph{Hilbert-Schmidt norm} $\|A\|_\HS := \|A\|_2$ and the \emph{operator norm} $\|A\|_\Op := \|A\|_\infty$ are two other important examples of Schatten norms. As trace-class operators, covariance operators are also \emph{Hilbert-Schmidt} ($\|A\|_\HS < \infty$) and bounded $(\|A\|_\Op < \infty)$, as Schatten norms satisfy $\|A\|_p \leq \|A\|_q$ for $p\geq q$. For $0<p<1$, $\|A\|_p$ does not produce a norm, but instead a \emph{quasi-norm}.

\paragraph{Reproducing kernel Hilbert spaces (RKHS).} Given a positive-definite kernel $K$, an associated RKHS $(\H, \langle \cdot, \cdot \rangle_\H)$ of bounded functions ($f\in\H$ satisfies $f(x)\leq C_x\|f\|_\H$ for some constant $C_x$) exists with the \emph{reproducing property} $\langle f, K(x,\cdot)\rangle_\H = f(x)$. This existence is given by the well-known Moore-Aronszajn theorem.

\subsection{Optimal Transport}
Let $(\M,d_\M)$ be a metric space equipped with a lower semi-continuous \emph{cost function} $c:\M \times \M \to \mathbb{R}_{\geq 0}$. Then, the optimal transport problem between two probability measures $\nu_0, \nu_1 \in \Probs(X)$ is given by
\begin{equation}
    \OT(\nu_0, \nu_1) \defeq \min_{\gamma\in \ADM(\nu_0,\nu_1)}\E_\gamma[c],
    \label{def:OT}
\end{equation}
where $\ADM(\nu_0,\nu_1)$ is the set of joint probabilities with marginals $\nu_0$ and $\nu_1$, and $\E_\nu[f]$ denotes the expected value of $f$ under $\nu$.

\paragraph{Wasserstein distance.} The $p$-Wasserstein distance $W_p$ between $\nu_0$ and $\nu_1$ is defined as
\begin{equation}
    W_p(\nu_0,\nu_1) \defeq \OT_{d_\M^p}(\nu_0, \nu_1)^{\frac{1}{p}},
\end{equation}
where $d_\M$ is the metric on $\M$ and $p\geq 1$. The case $p=2$ is particularly interesting, as the resulting metric is induced by a pseudo-Riemannian metric structure. One of the rare cases where the $2$-Wasserstein distance admits a closed form solution, is in th Euclidean case ($\M=\R^d$, $d_\M(x,y) = \|x-y\|$) between two multivariate Gaussian distributions $\nu_i=\N(\mu_i,\Sigma_i)$, $i=0,1$ which is given by \cite{givens84,dowson82,olkin82,knott84}
\begin{equation}\label{eq:gaus_was_dist}
\begin{aligned}
    W_2^2(\nu_0, \nu_1) =& \|\mu_0-\mu_1\|^2 + \Tr\left[\Sigma_0\right] + \Tr\left[\Sigma_1\right]\\
    &- 2 \Tr\left[\left(\Sigma_1^\frac{1}{2} \Sigma_0 \Sigma_1^\frac{1}{2}\right)^\frac{1}{2}\right].
\end{aligned}
\end{equation}
When $\mu_0=\mu_1=0$, we get as a special case a distance between the covariance matrices, which we denote by abuse of notation as 
\begin{equation}
    W_2(K_0,K_1) = W_2(\N(0,K_0), \N(0,K_1)).
\end{equation}

\paragraph{Entropic regularization.}
Let $\nu_0, \nu_1 \in \Probs(\M)$ have densities $p_{\nu_0}$ and $p_{\nu_1}$. Then, we denote by
\begin{equation}
    \KL(\nu_0 || \nu_1) \defeq - \E_{\nu_0}\left[\log\frac{p_{\nu_1}}{p_{\nu_1}}\right],
\end{equation}
the \emph{Kullback-Leibler divergence} (KL-divergence) between $\nu_0$ and $\nu_1$. Then, given $\epsilon > 0$, the entropic regularization of \eqref{def:OT}~\cite{cuturi13} is given by
\begin{equation}\label{eq:mainKL}
    \OT_c^\epsilon(\nu_0, \nu_1) \defeq \min_{\gamma\in \ADM(\nu_0,\nu_1)}\left\lbrace\E_\gamma[c]
    + \epsilon \KL(\gamma || \nu_0 \otimes \nu_1) \right\rbrace,
\end{equation}
which yields a strictly convex problem that is numerically more favorable to solve compared to \eqref{def:OT} due to the Sinkhorn-Knopp algorithm. For entropy-regularized Wasserstein distance, we use the notation
\begin{equation}
W_{p,\epsilon}(\nu_0,\nu_1) \defeq \left(\OT_{d_\M^p}^\epsilon(\nu_0,\nu_1)\right)^\frac{1}{p}.
\end{equation}

When $d_\M$ is the Euclidean distance, in the Gaussian special case $\nu_i=\Ncal(\mu_i,\Sigma_i)$ for $i=0,1$, a closed form expression can be derived~\cite{Mallasto2020entropyregularized,Janati2020entropicOT,GerGroGor19,del2020statistical}. Denote by
\begin{equation}
    M_\epsilon(\Sigma_i,\Sigma_j
    ) = -I + \left(I + \frac{16}{\epsilon^2}\Sigma_i^\frac{1}{2}\Sigma_j\Sigma_i^\frac{1}{2}\right)^\frac{1}{2},
\end{equation}
then
\begin{equation}\label{eq:gaus_ent_was_dist}
\begin{aligned}
    W_{2,\epsilon}^2(\nu_0, \nu_1)
    =& \|\mu_0 - \mu_1\|^2
    + \Tr\left[\Sigma_0\right] + \Tr\left[\Sigma_1\right]\\
    &+ \frac{\epsilon}{2}\log \det\left(I + \frac{1}{2} M_\epsilon(\Sigma_0,\Sigma_1)
    \right)\\
    &- \frac{\epsilon}{2}
    \Tr\left[M_\epsilon(\Sigma_0,\Sigma_1)\right]\\
\end{aligned}
\end{equation}
Note, that for computational reasons, it is advantageous to replace $K_0^\frac{1}{2}K_1K_0^\frac{1}{2}$ with $K_0K_1$ in $M_{01}^\epsilon$. The quantity remains the same, as the two matrices have the same eigenvalues. The invariance follows from the trace and determinant being functions of the eigenvalues of the input matrices.

\paragraph{Sinkhorn divergence.} The KL-divergence term in $\OT_c^\epsilon$ acts as a bias, as discussed in~\citep{feydy18}. This can be removed by defining the \emph{p-Sinkhorn divergence}~\citep{genevay17} as
\begin{equation}\label{def:sinkhorn}
    S_{p,\epsilon}(\nu_0, \nu_1) \defeq W_{p,\epsilon}^p(\nu_0, \nu_1) - \frac{1}{2}(W_{p,\epsilon}^p(\nu_0,\nu_0) + W_{p,\epsilon}^p(\nu_1,\nu_1) ).
\end{equation}

Especially, we have that $\lim\limits_{\epsilon\rightarrow 0}S_{2,\epsilon}(\nu_0, \nu_1) = W_2^2(\nu_0, \nu_1)$, and $\lim\limits_{\epsilon\rightarrow \infty} S_{2,\epsilon}(\nu_0, \nu_1) = \|\mu_0-\mu_1\|^2$. where $\mu_i$ is the mean of $\nu_i$. Therefore, the Sinkhorn divergence interpolates between OT and maximum mean discrepancy (MMD).

\paragraph{Extensions to GPs.} The $2$-Wasserstein distance and its entropic-regularization between normal distributions extend to the infinite-dimensional GP case~\cite{gelbrich90,mallasto17,haquang20entropic} readily by substituting the mean vector with the mean function and the covariance matrix with the covariance operator in \eqref{eq:gaus_was_dist} and \eqref{eq:gaus_ent_was_dist}, respectively. Thus for the entropic case, given two GPs $f_i\sim\GP(m_i, K_i)$, $i=0,1$, defined over an index space $(\X,\rho)$, the entropy-regularized $2$-Wasserstein distance between their kernels is given by
    \begin{equation}\label{eq:op_ent_was}
    \begin{aligned}
    W_{2,\epsilon}^2(K_0,K_1) =& \Tr\left[K_0\right] + \Tr\left[K_1\right]\\
    & + \frac{\epsilon}{2}\log \det\left(I + \frac{1}{2} M_\epsilon(K_0,K_1)\right)\\
    &- \frac{\epsilon}{2}
    \Tr\left[M_\epsilon(K_0,K_1)\right],
    \end{aligned}
    \end{equation}
and between the GPs
    \begin{equation}\label{eq:gp_ent_was}
    W_{2,\epsilon}^2(f_0,f_1) = \|m_0-m_1\|^2_{L^2(X,\rho)} + W_{2,\epsilon}^2(K_0,K_1).
    \end{equation}
The $2$-Sinkhorn divergences between the covariance operators and GPs can then be computed with \eqref{def:sinkhorn}.

\section{Marginal Complexity of Sinkhorn Divergence for Covariance Operators}
In this section we derive our main results on the complexity of computing the $2$-Sinkhorn divergence between two GPs using $n$ marginals. First we study the continuity of the entropy-regularized $2$-Wasserstein distance in Sec.~\ref{sec:continuity}. In Sec.~\ref{sec:empirical} we discuss empirical representations of the covariance operators, and in Sec.~\ref{sec:convergence} we show that the estimates computed in the finite-dimensional setting converge almost surely to the quantities in the infinite-dimensional setting. Finally, in Sec.~\ref{sec:marginal_complex} we provide results on the rate of convergence in the expected case (which we call \emph{marginal complexity}, cf. sample complexity) as well as concentration bounds for the estimation error. 

 We focus on the geometry induced on the covariance operators, as the Euclidean geometry $L^2(\X,\rho)$ of the mean functions is well-known. To briefly  recap, note that the means of the two GPs contribute $\|m_0-m_1\|_{L^2(\X,\rho)}^2$ in \eqref{eq:gp_ent_was}. This can be estimated by Monte Carlo integration~\cite{weinzierl2000introduction}, sampling $x_i$, $i=1,...,n$, I.I.D. from $\rho$ and estimating
\begin{equation}
    \begin{aligned}
    \|m_0-m_1\|_{L^2(\X,\rho)}^2 &= \int_\X \|m_0(x)-m_1(x)\|^2 d\rho(x)\\
    &\approx \frac{1}{n}\sum_{k=1}^n \|m_0(x_i)-m_1(x_i)\|^2,
    \end{aligned}
\end{equation}
whose error converges as $\mathcal{O}\left(\frac{1}{\sqrt{n}}\right)$ in expectation.

\subsection{Continuity Results}\label{sec:continuity}
We now focus on the continuity of \eqref{eq:op_ent_was}, which forms the backbone in Sec. \ref{sec:convergence} and \ref{sec:marginal_complex} for deriving error bounds and showing convergence of estimates computed using a finite amount of marginals of the GPs.

Start by splitting $W_{2,\epsilon}^2$ into two parts
\begin{equation}
    W_{2,\epsilon}^2(K_0,K_1) = \Tr \left[K_0 + K_1\right] - \frac{\epsilon}{2} F\left(\frac{16}{\epsilon^2}K_0K_1\right),
\end{equation}
where
\begin{equation}\label{eq:def_F}
\begin{aligned}
    F(A) \defeq&  \Tr\left[-I+(I+A)^\frac{1}{2}\right]\\
    &- \log\det\left(\frac{1}{2} + \frac{1}{2}(I+A)^\frac{1}{2}\right).
\end{aligned}
\end{equation}
The trace term is quite trivial, and we can study it without deriving any continuity results. In order to tackle the non-trivial $F$, we first introduce Kato's theorem.

\begin{theorem}[\citet{kato1987variation}]\label{thm:kato}
Let $\H$ be a separable Hilbert space, with $A,B$ self-adjoint compact operators. Let $\hat{\lambda}_i$ be an enumeration of eigenvalues of $A-B$. Then there exists extended (by zeros) enumerations $\lambda_i$ and $\lambda_i'$ of eigenvalues of $A$ and $B$, respectively, so that
\begin{equation}
\sum_i \phi(\lambda_i - \lambda_i') \leq \sum_i \phi(\hat{\lambda}_i),
\end{equation}
where $\phi$ is a non-negative convex function with $\phi(0)=0$.
\end{theorem}

With Kato's inequality, we can now show Lipschitzness.

\begin{proposition}\label{prop:total_continuity}
    $F$ is $\frac{1}{4}$-Lipschitz in the trace norm.
\end{proposition}
\begin{proof}
Let $A,B$ be covariance operators, and let $\{\lambda_i\}_{i\geq 0}$ and $\{\lambda_j'\}_{j\geq 0}$ be the (positive) eigenvalues of $A$ and $B$, respectively, and $\{\hat{\lambda}_k\}_{k\geq 0}$ be the eigenvalues of $A-B$. Then, by a straight-forward application of the triangle inequality, we get
\begin{equation}\label{eq:bound_rest}
    \begin{aligned}
    &\left| F(A)- F(B) \right|\\
    \leq& \inf\limits_\sigma \sum_i\left|\left((1+\lambda_i)^\frac{1}{2}
    -\log\left(1+(1+\lambda_i)^\frac{1}{2}\right)
    \right)\right.\\
    &-\left.\left((1+\lambda_{\sigma(i)}')^\frac{1}{2}
    -\log\left(1+(1+\lambda_{\sigma(i)}')^\frac{1}{2}\right)\right)
    \right|,\\
    \end{aligned}
\end{equation}
where $\sigma$ is a permutation of the indices. Now, looking at the function 
\begin{equation}
f(x) = x^\frac{1}{2} - \log(1+x^\frac{1}{2}),~ x\geq 1,
\end{equation}
we find it has the derivative
\begin{equation}
    f'(x) = \frac{1}{2(1+x^\frac{1}{2})},
\end{equation}
and so its growth can be bounded by
\begin{equation}\label{eq:real_bound_rest}
\left|f(x) - f(y) \right| \leq \sup\limits_{z\geq 1} \left|f'(z)\right| \left|x-y\right|\leq \frac{1}{4}\left|x-y\right|.
\end{equation}
Now applying \eqref{eq:real_bound_rest} and Theorem~\ref{thm:kato} to \eqref{eq:bound_rest} yields
\begin{equation}
\begin{aligned}
&\inf\limits_\sigma \sum_i \left|f(1+\lambda_i) - f(1+\lambda_{\sigma(i)}')\right|\\
\leq& \frac{1}{4}\inf \limits_\sigma \sum_i \left|\lambda_i - \lambda_{\sigma(i)}'\right| \\
\leq& \frac{1}{4}\sum_i \left|\hat{\lambda}_i\right| 
= \frac{1}{4}\|A-B\|_\Tr.
\end{aligned}
\end{equation}
\end{proof}

\begin{remark}\label{remark:sqrt_lipschitz}
Showing continuity similarly for the vanilla $2$-Wasserstein distance between covariance operators would fail, essentially as $x^\frac{1}{2}$, $x\geq 0$ is not Lipschitz. In contrast, entropic regularization provides us with smoothness in form of Lipschitzness and allows for using the theoretical machinery below to bound the error rate.
\end{remark}

However, the trace norm is difficult to work with, as it lacks smoothness. For example, known Hoeffding type bounds in Banach spaces require certain smoothness from the associated norm, which does not hold for the trace norm~\citep{pinelis1994optimum}. The Hilbert-Schmidt norm, resulting from an inner-product, is preferable, motivating the following bound.
\begin{lemma}\label{lem:tr_norm_bound}
Let $A,B,A',B'$ be positive, self-adjoint trace-class operators. Then,
\begin{equation}
    \|AB-A'B'\|_\Tr \leq \|B\|_\HS\|A-A'\|_\HS + \|A'\|_\HS\|B-B'\|_\HS.
\end{equation}
\end{lemma}
\begin{proof}
\begin{equation}
\begin{aligned}
    &\|AB-A'B'\|_\Tr\\
    =& \|AB - A'B + A'B - A'B'\|_\Tr\\
    \leq& \|(A-A')B\|_\Tr + \|A'(B-B')\|_\Tr\\
    \leq& \|B\|_\HS\|A-A'\|_\HS + \|A'\|_\HS\|B-B'\|_\HS.\\
\end{aligned}
\end{equation}
where we first use the trianlge inequality and then Hölder's inequality for Schatten norms (see e.g. \citep[Thm. 2.8]{simon2005trace}): for bounded operators $M,N$, any $p\in[1,\infty]$, and  $q$ so that $p^{-1}+q^{-1}=1$, the inequality
\begin{equation}
    \|MN\|_\Tr \leq \|M\|_p\|N\|_q,
\end{equation}
holds. We use the special case $p=q=2$.
\end{proof}

\subsection{Estimator for Entropic Wasserstein}\label{sec:empirical}
Before discussing how to estimate $W_{2,\epsilon}^2(K_0,K_1)$, we set the assumptions and some short-hand notation for the rest of this work.

\paragraph{Assumptions and notation.} For the following, let $K_i$ be positive and self-adjoint covariance operators over $L^2(\X,\rho)$, where $\X$ is assumed to be compact, and the kernels $K_i$ continuous. Furthermore,  let $X=(x_1,...,x_n) \subset \X$ be I.I.D. samples from $\rho$, let $\kappa_i = \sup\limits_{x,y}|K_i(x,y)|$, and define $K^x := K(\cdot, x)$.

\paragraph{Empirical operators.} Define the following operators
\begin{equation}
\begin{aligned}
T_{K} : f &\mapsto \int \langle f, K^{x} \rangle_\H K^{x} d\rho(x),\\
T_K^{(n)} : f &\mapsto \frac{1}{n}\sum_{i=1}^n \langle f, K^{x_i} \rangle_\H K^{x_i},\\
\left(\Sigma_K^{(n)}\right)_{ij} &= \frac{1}{n}K(x_i,x_j),\\
\end{aligned}
\end{equation}

where $T_K$ is a version of the operator $K$, with range and domain being the RKHS $(\H,\langle \cdot, \cdot \rangle)$ associated with $K$. Then, $T_K$ and $K$ share the same spectra up to zeros~\citep[Prop. 8]{rosasco10}. The operator $T_K^{(n)}$ is an empirical version of $T_K$, which we can represent in matrix form using $\Sigma_K^{(n)}$ in practical computations, as these two also share the same spectra~\citep[Prop. 9]{rosasco10}.

To define an empirical estimator of \eqref{eq:op_ent_was}, we first, by abuse of the notation introduced in~\eqref{eq:def_F}, define
\begin{equation}\label{eq:empirical_F}
\begin{aligned}
    F(K_0,K_1) &\defeq F\left(\frac{16}{\epsilon^2}K_0K_1\right)\\
    F(K_0,K_1,n) &\defeq F\left(\frac{16}{\epsilon^2} \Sigma_{K_0}^{(n)} \Sigma_{K_1}^{(n)}\right).
\end{aligned}
\end{equation}
Using this notation, the estimator for \eqref{eq:op_ent_was} is
\begin{equation}\label{eq:emp_op_ent_was}
    W^2_{2,\epsilon}(K_0,K_1,n) \defeq \Tr\left[ \Sigma_{K_0}^{(n)} + \Sigma_{K_1}^{(n)}\right] - \frac{\epsilon}{2} F\left(K_0, K_1, n\right),
\end{equation}
and by extension, the resulting empirical Sinkhorn divergence is denoted by $S_{2,\epsilon}(K_0,K_1,n)$. 

We start by bounding the difference between $W_{2,\epsilon}^2(K_0,K_1)$ and its estimator.
\begin{proposition}\label{prop:entropic_W_bound}
We have the upper bound
\begin{equation}\label{eq:entropic_W_bound}
\begin{aligned}
&|W_{2,\epsilon}^2(K_0,K_1)-W_{2,\epsilon}^2(K_0,K_1,n)|\\
\leq & |\Tr[T_{K_0} + T_{K_1}] - \Tr[T_{K_0}^{(n)}+T_{K_1}^{(n)}]|\\
& + \frac{2}{\epsilon}\left(\kappa_1\|T_{K_0}-T_{K_0}^{(n)}\|_\HS + \kappa_0\|T_{K_1}-T_{K_1}^{(n)}\|_\HS\right)
\end{aligned}
\end{equation}
\end{proposition}
\begin{proof}
Using triangle inequality, we get
\begin{equation}\label{eq:entropic_W_bound1}
    \begin{aligned}
    &\left|W_{2,\epsilon}^2(K_0,K_1)-W_{2,\epsilon}^2(K_0,K_1,n)\right|\\
    \leq & \left|\Tr[T_{K_0} + T_{K_1}] - \Tr\left[T_{K_0}^{(n)}+T_{K_1}^{(n)}\right]\right|\\
    & + \frac{\epsilon}{2}\left|F\left(K_0,K_1\right)- F\left(K_0,K_1,n\right)\right|.
    \end{aligned}
\end{equation}
Now focus on the second term in \eqref{eq:entropic_W_bound1}. Note that we can replace $K_0,K_1$ with $T_{K_0},T_{K_1}$ without changing the quantities, as these operators have the same spectra up to zeros. The same holds when replacing $\Sigma_{K_0}^{(n)},\Sigma_{K_1}^{(n)}$ with $T_{K_0},T_{K_1}$. Then, applying Proposition~\ref{prop:total_continuity} and Lemma~\ref{lem:tr_norm_bound} with the bounds $\|T_{K_i}\|_\HS\leq \kappa_i$, $\|T_{K_i}^{(n)}\|_\HS \leq \kappa_i$, we get
\begin{equation}
\begin{aligned}
 &\frac{\epsilon}{2}\left|F\left(K_0,K_1\right)- F\left(K_0,K_1,n\right)\right|\\
=&\frac{\epsilon}{2}\left|F\left(\frac{16}{\epsilon^2}K_0K_1\right)- F\left(\frac{16}{\epsilon^2}\Sigma_{K_0}^{(n)}\Sigma_{K_1}^{(n)}\right)\right|\\
\leq &\frac{2}{\epsilon}\left\|T_{K_0}T_{K_1} - T^{(n)}_{K_0}T^{(n)}_{K_1}\right\|_\Tr\\
\leq & \frac{2}{\epsilon}\left(
\kappa_1\left\|T_{K_0} - T^{(n)}_{K_0}\right\|_\HS + \kappa_0\left\|T_{K_1} - T^{(n)}_{K_1}\right\|_\HS
\right).
\end{aligned}
\end{equation}
\end{proof}

In order to emphasize that we are working in a stochastic setting, we introduce the following mean-zero random variables
\begin{equation}
\begin{aligned}
\xi_k &= \Tr\left[T_{K_0}+T_{K_1}\right] - (K_0(x_k,x_k)+K_1(x_k,x_k)),\\
\xi^i_k &= T_{K_i} - \langle \cdot, K_i^{x_k} \rangle_{\H_i}K_i^{x_k},
\end{aligned}
\end{equation}
where $x_k$ is sampled from $\rho$, and we have the bounds $|\xi_k|\leq 2(\kappa_0+\kappa_1)$ and $\|\xi_k^i\|_\HS \leq 2\kappa_i$ (see \citep{rosasco10} Theorem 7 and Proposition 11). Then, we have
\begin{equation}
\begin{aligned}
    \frac{1}{n}\sum_{k=1}^n \xi_k  &= \Tr[K_0 + K_1] - \Tr[T_{K_0}^{(n)}+T_{K_1}^{(n)}],\\
    \frac{1}{n}\sum_{k=1}^n \xi_k^i &= T_{K_i} - T_{K_i}^{(n)}.
\end{aligned}
\end{equation}
Now, we write the upper bound in Proposition \eqref{prop:entropic_W_bound} as 

\begin{equation}\label{eq:entropic_W_bound_rv}
\begin{aligned}
&|W_{2,\epsilon}^2(K_0,K_1)-W_{2,\epsilon}^2(K_0,K_1,n)|\\
\leq & \frac{2}{\epsilon}\left(\kappa_1\left\|\frac{1}{n}\sum_{k=1}^n \xi_k^0\right\|_\HS + \kappa_0\left\|\frac{1}{n}\sum_{k=1}^n \xi_k^1\right\|_\HS\right) \\
& + \left|\frac{1}{n}\sum_{k=1}^n \xi_k\right|.
\end{aligned}
\end{equation}

With similar computations, we also get 
\begin{equation}\label{eq:sinkhorn_bound_rv}
\begin{aligned}
&|S_{2,\epsilon}(K_0,K_1)-S_{2,\epsilon}(K_0,K_1,n)|\\
\leq & \frac{2(\kappa_0+\kappa_1)}{\epsilon}\left(\left\|\frac{1}{n}\sum_{k=1}^n \xi_k^0\right\|_\HS + \left\|\frac{1}{n}\sum_{k=1}^n \xi_k^1\right\|_\HS\right).
\end{aligned}
\end{equation}

\subsection{Convergence of Estimates}\label{sec:convergence}
We now recall the law of large numbers for Hilbert spaces, which is then utilized to show convergence of \eqref{eq:emp_op_ent_was} to the right quantity.

\begin{theorem}[Law of Large Numbers for Hilbert Spaces~\cite{chen11}]\label{thm:loln_hilbert}
Let $\xi_1,...,\xi_n$ be I.I.D $(\H,\|\cdot\|)$-valued random variables with $\Expect\|\xi_1\| < \infty$. Then,
\begin{equation}
    \lim\limits_{n\to \infty} \frac{1}{n} \sum_{k=1}^n \xi_k  = \Expect[\xi_1],
\end{equation}
almost surely.
\end{theorem}

The following result then immediately follows.

\begin{theorem}\label{thm:convergence}
Let $K_0, K_1$ be covariance operators and $X=\{x_1,...,x_n\}$ be I.I.D. samples from $(\X,\rho)$. Then,
\begin{equation}
\begin{aligned}
    \lim\limits_{n\to\infty} W_{2,\epsilon}^2(K_0,K_1,n) &= W_{2,\epsilon}^2(K_0,K_1),\\
        \lim\limits_{n\to \infty}S_{2,\epsilon}(K_0,K_1,n) &=  S_{2,\epsilon}(K_0,K_1),
\end{aligned}
\end{equation}
almost surely.
\end{theorem}
\begin{proof}
Apply Theorem~\ref{thm:loln_hilbert} on the three terms in~\eqref{eq:entropic_W_bound_rv}
\end{proof}

\subsection{Marginal Complexities}\label{sec:marginal_complex}
Next, we study the rate of convergence guaranteed by Theorem~\ref{thm:convergence}. We consider the entropy-regularized $2$-Wasserstein distance and the $2$-Sinkhorn divergence. In both cases, we compute the rate of convergence for the expected error, as well as provide concentration results on the error by invoking Hoeffding's inequality.

\begin{theorem}[Hoeffding]\label{thm:hoeffding}
Let $\xi_1,...,\xi_n$ be zero-mean independent random elements of a Hilbert space $(\H, \|\cdot\|)$ so that $\|\xi_k\| \leq C$, $k=1,..,n$. Then,
\begin{equation}
    \Pr\left\lbrace
    \left\|\frac{1}{n}\sum_{k=1}^n \xi_k\right\| \geq t
    \right\rbrace \leq 2\exp\left(-\frac{nt^2}{2C^2}\right).
\end{equation}
\end{theorem}

\paragraph{Marginal complexity of entropy-regularized Wasserstein.} 
We start by looking at the expected error stemming from using a finite amount of marginals.

\begin{theorem}[Entropic Wasserstein Marginal Complexity]\label{thm:average_bound_ot}
\begin{equation}\label{eq:was_expected_bound}
    \begin{aligned}
    &\Expect_\rho\left|W_{2,\epsilon}^2(K_0,K_1) - W_{2,\epsilon}^2(K_0,K_1,n)\right|\\
    \leq & \frac{1}{\sqrt{n}}\left(\kappa_0 + \kappa_1 + \frac{8}{\epsilon}\kappa_0\kappa_1\right)
    \end{aligned}
\end{equation}
\end{theorem}
\begin{proof}
Use the upper bound \eqref{eq:entropic_W_bound_rv} and take expectations on both sides. Then,
\begin{equation}
\begin{aligned}
    \Expect_\rho\left[\left| 
    \frac{1}{n}\sum_{k=1}^n\xi_k 
    \right|\right] &\leq \sqrt{\Var\left(\frac{1}{n}\sum_{k=1}^n \xi_k \right)}\\
    &=  \sqrt{\frac{1}{n^2}\sum_{k=1}^n\Var(\xi_k)}\\
    &\leq \frac{\kappa_0+\kappa_1}{\sqrt{n}},\\
\end{aligned}
\end{equation}
where the first inequality follows from Jensen's inequality and $\xi_k$ having mean zero, the first equality from $x_k$ being I.I.D and thus $\xi_k$ are I.I.D. The last inequality follows from $0\leq \xi_k \leq 2(\kappa_0+\kappa_1)$ almost surely, and so by Popoviciu's inequality on variances we get $\Var(\xi_k) \leq \frac{1}{4}[2(\kappa_0+\kappa_1)]^2$.

For the other two terms in \eqref{eq:entropic_W_bound_rv}, compute
\begin{equation}\label{eq:thm_average_bound_ot3}
\begin{aligned}
    \Expect_\rho\left[\left\|\frac{1}{n}\sum_{k=1}^n \xi_k^i\right\|_\HS\right] &\leq \sqrt{\Expect_\rho\left[\left\|\frac{1}{n}\sum_{k=1}^n \xi_k^i\right\|^2_\HS\right]}\\
    &\leq \sqrt{\frac{1}{n^2}\sum_{k=1}^n\Expect_\rho\left[\left\| \xi_k^i\right\|^2_\HS\right]}\\
    &\leq \frac{2\kappa_i}{\sqrt{n}},
\end{aligned}
\end{equation}
where we first use Jensen's inequality. Then, as $\xi_k^i$ are I.I.D and centered, they are orthogonal with respect to $\Expect\langle \cdot, \cdot\rangle$, and thus we can move the sum outside of the squared norm. Finally, we use $\|\xi_k^i\|_\HS\leq 2\kappa_i$, and the claim follows.
\end{proof}

Next, we apply Hoeffding's inequality to derive a concentration bound for the error.

\begin{theorem}[Entropic Wasserstein Error Concentration]\label{thm:hoeffding_bound_ot}
With probability at least $1-6e^{-\theta}$
\begin{equation}\label{eq:r_was_hoeffding_bound}
    \begin{aligned}
    &\left|W_{2,\epsilon}^2(K_0,K_1) - W_{2,\epsilon}^2(K_0,K_1,n)\right|\\
    \leq & \sqrt{\frac{(72(\kappa_0+\kappa_1)^2 + 288\epsilon^{-2}\kappa_0^2\kappa_1^2)\theta}{n}},
    \end{aligned}
\end{equation}
\end{theorem}

\begin{proof}
Let $Z_n = W_{2,\epsilon}^2(K_0,K_1) - W_{2,\epsilon}^2(K_0,K_1,n)$, and note that for positive random variables $X_j$, $j=1,...,m$, we have the bound
\begin{equation}\label{eq:sum_x_geq_t}
    \Pr\left\lbrace \sum_{j=1}^m X_j \geq t \right\rbrace \leq \sum_{j=1}^m \Pr\left\lbrace X_j \geq \frac{t}{m}
    \right\rbrace.
\end{equation}
Combining \eqref{eq:sum_x_geq_t} with the bound \eqref{eq:entropic_W_bound_rv}, we get
\begin{equation}
\begin{aligned}
    &\Pr\{|Z_n|\geq t\} \leq \Pr\{T + T_1 + T_2 \geq t\}\\
    \leq &   \Pr\left\lbrace \left|\frac{1}{n}\sum_{k=1}^n \xi_k\right|\geq \frac{t}{3}\right\rbrace
    + \Pr\left\lbrace \frac{2\kappa_1}{\epsilon}\left\|\frac{1}{n}\sum_{k=1}^n \xi_k^0\right\|_\HS\geq \frac{t}{3}\right\rbrace\\
    &+ \Pr\left\lbrace \frac{2\kappa_0}{\epsilon}\left\|\frac{1}{n}\sum_{k=1}^n \xi_k^1\right\|_\HS\geq \frac{t}{3}\right\rbrace\\
    \defeq & P(t) + P_0(t) + P_1(t).
\end{aligned}
\end{equation}
These terms can further be bounded from above using Hoeffding's inequality (Theorem~\ref{thm:hoeffding})
\begin{equation}
    P(t) \leq 2\exp\left(-\frac{nt^2}{72(\kappa_0+\kappa_1)^2}\right).
\end{equation}

\begin{equation}
    P_i(t)\leq 2\exp\left(-\frac{nt^2}{288\epsilon^{-2}\kappa_0^2\kappa_1^2}\right).
\end{equation}

Combining these bounds, we have
\begin{equation}
    \Pr\{|Z_n| \geq t\} \leq 6\exp\left(-\frac{nt^2}{72(\kappa_0+\kappa_1)^2 + 288\epsilon^{-2}\kappa_0^2\kappa_1^2}\right).
\end{equation}

Choose
\begin{equation}
    \theta = \frac{nt^2}{72(\kappa_0+\kappa_1)^2 + 288\epsilon^{-2}\kappa_0^2\kappa_1^2},
\end{equation}
and the result follows.
\end{proof}

\paragraph{Marginal complexity of Sinkhorn.} Next, we provide the marginal complexity for the Sinkhorn divergence between covariance operators, which can be shown in a similar fashion to Theorems~\ref{thm:average_bound_ot} and \ref{thm:hoeffding_bound_ot}.

\begin{theorem}[Sinkhorn Marginal Complexity]\label{thm:average_bound_sinkhorn}
The expected estimation error of the $2$-Sinkhorn divergence between covariance operators is given by
\begin{equation}\label{eq:sinkhorn_expected_bound}
    \begin{aligned}
    &\Expect_\rho\left|S_{2,\epsilon}(K_0,K_1) - S_{2,\epsilon}(K_0,K_1,n)\right|\\
    \leq & \frac{4(\kappa_0+\kappa_1)^2}{\epsilon \sqrt{n}}.
    \end{aligned}
\end{equation}
\end{theorem}
\begin{proof}
Take expectations of both sides in \eqref{eq:sinkhorn_bound_rv} and apply \eqref{eq:thm_average_bound_ot3}.
\end{proof}

\begin{theorem}[Sinkhorn Error Concentration]\label{thm:hoeffding_bound_sinkhorn}
With probability at least $1-4e^{-\theta}$,
\begin{equation}
    \begin{aligned}
    &\left| S_{2,\epsilon}(K_0,K_1)-S_{2,\epsilon}(K_0,K_1,n)\right|\\
    \leq & \frac{8\sqrt{2}(\kappa_0+\kappa_1)^2}{\epsilon}\sqrt{\frac{\theta}{n}}.
    \end{aligned}
\end{equation}
\end{theorem}
\begin{proof}
Similar as the proof of Theorem~\ref{thm:hoeffding_bound_ot}, except use the bound in \eqref{eq:sinkhorn_bound_rv}.
\end{proof}

\begin{figure*}
    \centering
    \includegraphics[width=1\linewidth]{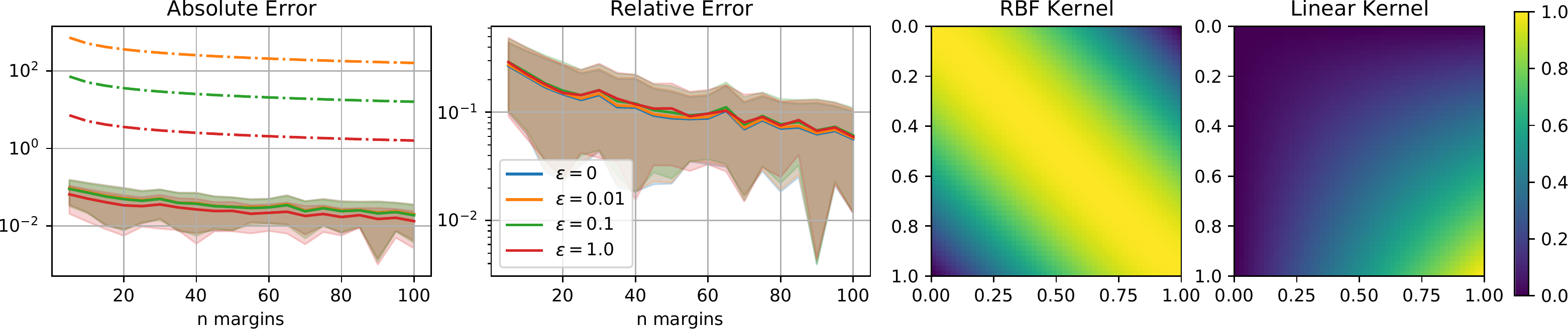}
    \caption{Absolute and relative error for the estimated Sinkhorn divergence between an RBF and linear kernel on $[0,1]$ as the amount of margins is increased. Solid lines represent the average for random sampling (shaded area illustrates $\sigma$-confidence intervals), and dot-dashed lines are upper bounds for the expected errors. The blue curve in the first plot is hidden behind the yellow one.}
    \label{fig:rbs_vs_linear}
\end{figure*}

\begin{figure*}
    \centering
    \includegraphics[width=.9\linewidth]{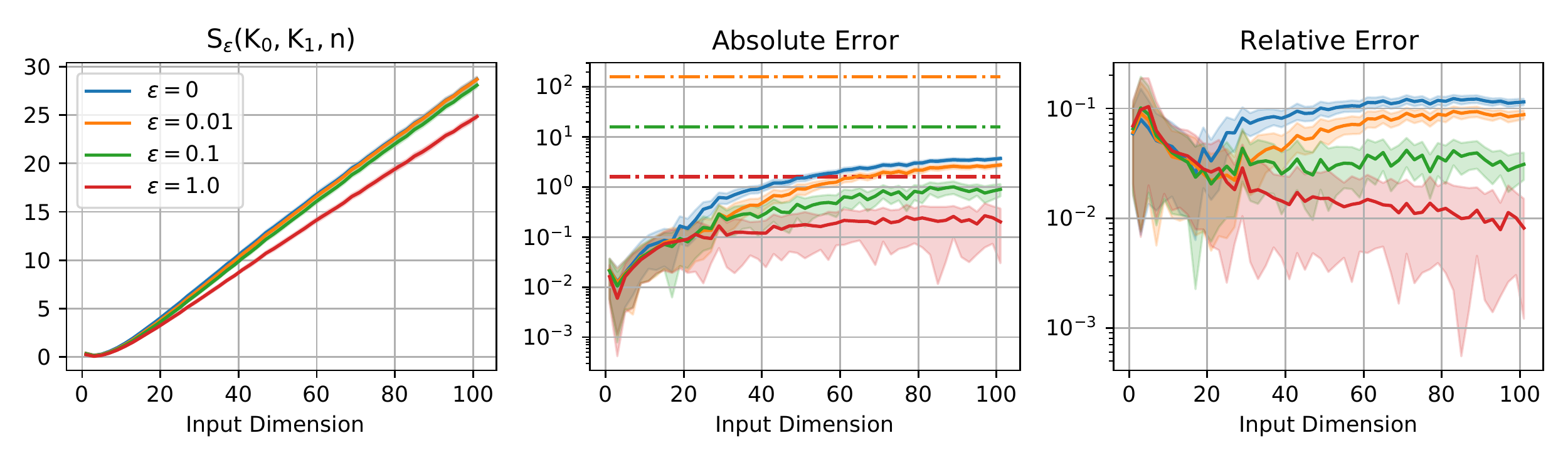}
    \caption{Absolute error and relative error for the estimated Sinkhorn divergence between an RBF and a linear kernel on $[0,1]^d$, where $d$ is the input dimension. The dot-dashed lines are the upper bounds for the expected errors.}
    \label{fig:errors_vs_dim}
\end{figure*}

\section{Experiments}
We now empirically illustrate the convergence by studying the behavior of the $2$-Sinkhorn divergence as we vary the entropic regularization $\epsilon$, amount $n$ of marginals used, and the input dimension of the kernels.

\paragraph{Varying amount of marginals.} We compute the divergence between the radial basis function (RBF) kernel with parameters $(\sigma^2, \lambda)=(1,1)$ and linear kernel on $\X=[0,1]$ with a varying amount of marginals, sampled randomly from $\X$ ($m=50$ samples from the uniform distribution). The error is estimated by computing a 'ground truth' using $n=1000$ uniformly sampled marginals.

The results in Fig.~\ref{fig:rbs_vs_linear} illustrate that as we increase $\epsilon$, the absolute error
\begin{equation}
    \left|S_{2,\epsilon}(K_0,K_1) - S_{2,\epsilon}(K_0,K_1,n)\right|
\end{equation}
goes down as expected from Theorem~\ref{thm:average_bound_sinkhorn}. However, the relative errors
\begin{equation}
    \frac{\left|S_{2,\epsilon}(K_0,K_1) - S_{2,\epsilon}(K_0,K_1,n)\right|}{S_{2,\epsilon}(K_0,K_1)}
\end{equation}
stay almost identical. This is perhaps a bit disappointing, as one might hope $\epsilon$ to increase the relative accuracy too.

\paragraph{Varying input dimension.} In this experiment, we demonstrate the effect of the input dimension (dimensionality of $\X$) on the Sinkhorn divergence. We again consider the RBF kernel, with unit variances and length scales, and the linear kernel. For each dimension, we sample $n=100$ marginals $m=50$ times, and compare the resulting divergences to a 'ground truth' computed with $n=1000$ samples.

As can be seen in Fig.~\ref{fig:errors_vs_dim}, increasing the input dimension also increases the divergence between the kernels, seemingly in an unbouded manner. However, the absolute errors still remain bounded below the theoretical average error bounds, with increasing $\epsilon$ reducing the error. The interesting part is the behavior of the relative error as the dimension increases, as now we witness a reduction as we increase $\epsilon$, in contrast to the experiment with increasing marginals.  

\section{Conclusion}
We have shown that the expected error of estimating the $2$-Sinkhorn divergence between GPs using $n$ marginals behaves as $\mathcal{O}\left(\epsilon^{-1}n^{-\frac{1}{2}}\right)$, implying that entropic regularization helps estimating the divergence. However, as $\epsilon$ increases, the Sinkhorn divergence converges to MMD, which is zero between covariance operators, and thus the divergence keeps decreasing. Therefore it is not surprising, that also the absolute errors should decrease. 

Therefore the interesting quantity is the relative error of the estimate, which entropic regularization seems to help only as we increase the input dimension. This aligns well with what is known about OT and MMD: OT suffers from the curse of dimensionality, where as MMD benefits from dimension-free convergence rates. Notably, the error rates provided in the entropy-regularize case (Theorems~\ref{thm:average_bound_sinkhorn} and \ref{thm:average_bound_ot}) are dimension-free, and the demonstrations imply that the larger the regularization, the better the behavior with respect to input dimension.

\section*{Acknowledgements}
The author would like to thank Augusto Gerolin and H\`a Quang Minh for their comments and feedback on the project. This work was supported by Academy of Finland (Flagship programme: Finnish Center for Artificial Intelligence FCAI, Grant 328400). Aalto Science-IT project is acknowledged for the computational resoruces provided.

\bibliography{references}

\begin{thebibliography}{40}
\providecommand{\natexlab}[1]{#1}
\providecommand{\url}[1]{\texttt{#1}}
\expandafter\ifx\csname urlstyle\endcsname\relax
  \providecommand{\doi}[1]{doi: #1}\else
  \providecommand{\doi}{doi: \begingroup \urlstyle{rm}\Url}\fi

\bibitem[Arsigny et~al.(2006)Arsigny, Fillard, Pennec, and
  Ayache]{arsigny2006log}
Arsigny, V., Fillard, P., Pennec, X., and Ayache, N.
\newblock Log-{E}uclidean metrics for fast and simple calculus on diffusion
  tensors.
\newblock \emph{Magnetic Resonance in Medicine: An Official Journal of the
  International Society for Magnetic Resonance in Medicine}, 56\penalty0
  (2):\penalty0 411--421, 2006.

\bibitem[Bhatia(1997)]{Bhatia:1997Matrix}
Bhatia, R.
\newblock \emph{Matrix Analysis}.
\newblock Springer, 1997.

\bibitem[Chen \& Zhu(2011)Chen and Zhu]{chen11}
Chen, Y.-X. and Zhu, W.-J.
\newblock Note on the strong law of large numbers in a {H}ilbert space.
\newblock \emph{Gen. Math}, 19\penalty0 (3):\penalty0 11--18, 2011.

\bibitem[Cuturi(2013)]{cuturi13}
Cuturi, M.
\newblock Sinkhorn distances: {L}ightspeed computation of optimal transport.
\newblock In \emph{Advances in neural information processing systems}, pp.\
  2292--2300, 2013.

\bibitem[del Barrio \& Loubes(2020)del Barrio and Loubes]{del2020statistical}
del Barrio, E. and Loubes, J.-M.
\newblock The statistical effect of entropic regularization in optimal
  transportation.
\newblock \emph{arXiv preprint arXiv:2006.05199}, 2020.

\bibitem[Dessein et~al.(2018)Dessein, Papadakis, and Rouas]{dessein18}
Dessein, A., Papadakis, N., and Rouas, J.-L.
\newblock Regularized optimal transport and the rot mover's distance.
\newblock \emph{The Journal of Machine Learning Research}, 19\penalty0
  (1):\penalty0 590--642, 2018.

\bibitem[Di~Marino \& Gerolin(2020)Di~Marino and Gerolin]{dimarino20}
Di~Marino, S. and Gerolin, A.
\newblock Optimal transport losses and {S}inkhorn algorithm with general convex
  regularization.
\newblock \emph{arXiv preprint arXiv:2007.00976}, 2020.

\bibitem[Dowson \& Landau(1982)Dowson and Landau]{dowson82}
Dowson, D. and Landau, B.
\newblock The {F}r{\'e}chet distance between multivariate normal distributions.
\newblock \emph{Journal of multivariate analysis}, 12\penalty0 (3):\penalty0
  450--455, 1982.

\bibitem[Dudley(1969)]{dudley1969speed}
Dudley, R.~M.
\newblock The speed of mean {G}livenko-{C}antelli convergence.
\newblock \emph{The Annals of Mathematical Statistics}, 40\penalty0
  (1):\penalty0 40--50, 1969.

\bibitem[Faraki et~al.(2015)Faraki, Harandi, and
  Porikli]{faraki2015approximate}
Faraki, M., Harandi, M.~T., and Porikli, F.
\newblock Approximate infinite-dimensional region covariance descriptors for
  image classification.
\newblock In \emph{2015 IEEE international conference on acoustics, speech and
  signal processing (ICASSP)}, pp.\  1364--1368. IEEE, 2015.

\bibitem[Feydy et~al.(2019)Feydy, S{\'e}journ{\'e}, Vialard, Amari, Trouv{\'e},
  and Peyr{\'e}]{feydy18}
Feydy, J., S{\'e}journ{\'e}, T., Vialard, F.-X., Amari, S.-i., Trouv{\'e}, A.,
  and Peyr{\'e}, G.
\newblock Interpolating between optimal transport and {MMD} using {S}inkhorn
  divergences.
\newblock In \emph{The 22nd International Conference on Artificial Intelligence
  and Statistics}, pp.\  2681--2690. PMLR, 2019.

\bibitem[Gelbrich(1990)]{gelbrich90}
Gelbrich, M.
\newblock On a formula for the {L}2 {W}asserstein metric between measures on
  {E}uclidean and {H}ilbert spaces.
\newblock \emph{Mathematische Nachrichten}, 147\penalty0 (1):\penalty0
  185--203, 1990.

\bibitem[Genevay et~al.(2018)Genevay, Peyre, and Cuturi]{genevay17}
Genevay, A., Peyre, G., and Cuturi, M.
\newblock Learning generative models with {S}inkhorn divergences.
\newblock In Storkey, A. and Perez-Cruz, F. (eds.), \emph{Proceedings of the
  Twenty-First International Conference on Artificial Intelligence and
  Statistics}, volume~84 of \emph{Proceedings of Machine Learning Research},
  pp.\  1608--1617, 2018.

\bibitem[Genevay et~al.(2019)Genevay, Chizat, Bach, Cuturi, and
  Peyr\'{e}]{genevay18sample}
Genevay, A., Chizat, L., Bach, F., Cuturi, M., and Peyr\'{e}, G.
\newblock Sample complexity of {S}inkhorn divergences.
\newblock In Chaudhuri, K. and Sugiyama, M. (eds.), \emph{Proceedings of
  Machine Learning Research}, volume~89 of \emph{Proceedings of Machine
  Learning Research}, pp.\  1574--1583, 2019.

\bibitem[Gerolin et~al.(2019)Gerolin, Grossi, and Gori-Giorgi]{GerGroGor19}
Gerolin, A., Grossi, J., and Gori-Giorgi, P.
\newblock Kinetic correlation functionals from the entropic regularisation of
  the strictly-correlated electrons problem.
\newblock \emph{arXiv:1911.05818}, 2019.

\bibitem[Givens et~al.(1984)Givens, Shortt, et~al.]{givens84}
Givens, C.~R., Shortt, R.~M., et~al.
\newblock A class of {W}asserstein metrics for probability distributions.
\newblock \emph{The Michigan Mathematical Journal}, 31\penalty0 (2):\penalty0
  231--240, 1984.

\bibitem[Harandi et~al.(2014)Harandi, Salzmann, and
  Porikli]{harandi2014bregman}
Harandi, M., Salzmann, M., and Porikli, F.
\newblock Bregman divergences for infinite dimensional covariance matrices.
\newblock In \emph{Proceedings of the IEEE Conference on Computer Vision and
  Pattern Recognition}, pp.\  1003--1010, 2014.

\bibitem[Hauberg et~al.(2015)Hauberg, Schober, Liptrot, Hennig, and
  Feragen]{hauberg2015random}
Hauberg, S., Schober, M., Liptrot, M., Hennig, P., and Feragen, A.
\newblock A random {R}iemannian metric for probabilistic shortest-path
  tractography.
\newblock In \emph{International Conference on Medical Image Computing and
  Computer-Assisted Intervention}, pp.\  597--604. Springer, 2015.

\bibitem[Janati et~al.(2020)Janati, Muzellec, Peyr{\'e}, and
  Cuturi]{Janati2020entropicOT}
Janati, H., Muzellec, B., Peyr{\'e}, G., and Cuturi, M.
\newblock Entropic optimal transport between (unbalanced) {Gaussian} measures
  has a closed form.
\newblock \emph{arXiv preprint arXiv:2006.02572}, 2020.

\bibitem[Kato(1987)]{kato1987variation}
Kato, T.
\newblock Variation of discrete spectra.
\newblock \emph{Communications in Mathematical Physics}, 111\penalty0
  (3):\penalty0 501--504, 1987.

\bibitem[Knott \& Smith(1984)Knott and Smith]{knott84}
Knott, M. and Smith, C.~S.
\newblock On the optimal mapping of distributions.
\newblock \emph{Journal of Optimization Theory and Applications}, 43\penalty0
  (1):\penalty0 39--49, 1984.

\bibitem[L{\^e} et~al.(2015)L{\^e}, Unkelbach, Ayache, and
  Delingette]{le2015gpssi}
L{\^e}, M., Unkelbach, J., Ayache, N., and Delingette, H.
\newblock Gpssi: {G}aussian process for sampling segmentations of images.
\newblock In \emph{International Conference on Medical Image Computing and
  Computer-Assisted Intervention}, pp.\  38--46. Springer, 2015.

\bibitem[Mallasto \& Feragen(2017)Mallasto and Feragen]{mallasto17}
Mallasto, A. and Feragen, A.
\newblock Learning from uncertain curves: The 2-{W}asserstein metric for
  {G}aussian processes.
\newblock In \emph{Advances in Neural Information Processing Systems}, pp.\
  5660--5670, 2017.

\bibitem[Mallasto et~al.(2020)Mallasto, Gerolin, and
  Minh]{Mallasto2020entropyregularized}
Mallasto, A., Gerolin, A., and Minh, H.
\newblock Entropy-regularized 2-{Wasserstein} distance between {Gaussian}
  measures.
\newblock \emph{preprint arXiv:2006.03416}, 2020.

\bibitem[Masarotto et~al.(2019)Masarotto, Panaretos, and
  Zemel]{masarotto2019procrustes}
Masarotto, V., Panaretos, V.~M., and Zemel, Y.
\newblock Procrustes metrics on covariance operators and optimal transportation
  of gaussian processes.
\newblock \emph{Sankhya A}, 81\penalty0 (1):\penalty0 172--213, 2019.

\bibitem[Mena \& Weed(2019)Mena and Weed]{mena2019statistical}
Mena, G. and Weed, J.
\newblock Statistical bounds for entropic optimal transport: sample complexity
  and the central limit theorem.
\newblock \emph{arXiv preprint arXiv:1905.11882}, 2019.

\bibitem[Minh(2015)]{haquang2015affine}
Minh, H.~Q.
\newblock Affine-invariant {R}iemannian distance between infinite-dimensional
  covariance operators.
\newblock In \emph{International Conference on Geometric Science of
  Information}, pp.\  30--38. Springer, 2015.

\bibitem[Minh(2020)]{haquang20entropic}
Minh, H.~Q.
\newblock Entropic regularization of {W}asserstein distance between
  infinite-dimensional {G}aussian measures and {G}aussian processes.
\newblock \emph{arXiv preprint arXiv:2011.07489}, 2020.

\bibitem[Minh et~al.(2014)Minh, San~Biagio, and Murino]{haquang2014log}
Minh, H.~Q., San~Biagio, M., and Murino, V.
\newblock {L}og-{H}ilbert-{S}chmidt metric between positive definite operators
  on {H}ilbert spaces.
\newblock In \emph{Advances in neural information processing systems}, pp.\
  388--396, 2014.

\bibitem[Olkin \& Pukelsheim(1982)Olkin and Pukelsheim]{olkin82}
Olkin, I. and Pukelsheim, F.
\newblock The distance between two random vectors with given dispersion
  matrices.
\newblock \emph{Linear Algebra and its Applications}, 48:\penalty0 257--263,
  1982.

\bibitem[Pennec et~al.(2006)Pennec, Fillard, and Ayache]{pennec2006riemannian}
Pennec, X., Fillard, P., and Ayache, N.
\newblock A {R}iemannian framework for tensor computing.
\newblock \emph{International Journal of computer vision}, 66\penalty0
  (1):\penalty0 41--66, 2006.

\bibitem[Peyr{\'e} et~al.(2019)Peyr{\'e}, Cuturi,
  et~al.]{peyre2019computational}
Peyr{\'e}, G., Cuturi, M., et~al.
\newblock Computational optimal transport: With applications to data science.
\newblock \emph{Foundations and Trends{\textregistered} in Machine Learning},
  11\penalty0 (5-6):\penalty0 355--607, 2019.

\bibitem[Pigoli et~al.(2014)Pigoli, Aston, Dryden, and
  Secchi]{pigoli2014distances}
Pigoli, D., Aston, J.~A., Dryden, I.~L., and Secchi, P.
\newblock Distances and inference for covariance operators.
\newblock \emph{Biometrika}, 101\penalty0 (2):\penalty0 409--422, 2014.

\bibitem[Pinelis(1994)]{pinelis1994optimum}
Pinelis, I.
\newblock Optimum bounds for the distributions of martingales in {B}anach
  spaces.
\newblock \emph{The Annals of Probability}, pp.\  1679--1706, 1994.

\bibitem[Ripani(2017)]{RipThesis}
Ripani, L.
\newblock The {S}chr\"odinger problem and its links to optimal transport and
  functional inequalities.
\newblock Ph.D. thesis, University Lyon 1, 2017.

\bibitem[Roberts et~al.(2013)Roberts, Osborne, Ebden, Reece, Gibson, and
  Aigrain]{roberts2013gaussian}
Roberts, S., Osborne, M., Ebden, M., Reece, S., Gibson, N., and Aigrain, S.
\newblock Gaussian processes for time-series modelling.
\newblock \emph{Philosophical Transactions of the Royal Society A:
  Mathematical, Physical and Engineering Sciences}, 371\penalty0
  (1984):\penalty0 20110550, 2013.

\bibitem[Rosasco et~al.(2010)Rosasco, Belkin, and De~Vito]{rosasco10}
Rosasco, L., Belkin, M., and De~Vito, E.
\newblock On learning with integral operators.
\newblock \emph{Journal of Machine Learning Research}, 11\penalty0
  (30):\penalty0 905--934, 2010.

\bibitem[Simon(2005)]{simon2005trace}
Simon, B.
\newblock \emph{Trace ideals and their applications}.
\newblock Number 120. American Mathematical Soc., 2005.

\bibitem[Villani(2008)]{villani08}
Villani, C.
\newblock \emph{Optimal transport: old and new}, volume 338.
\newblock Springer Science \& Business Media, 2008.

\bibitem[Weinzierl(2000)]{weinzierl2000introduction}
Weinzierl, S.
\newblock Introduction to {M}onte {C}arlo methods.
\newblock \emph{arXiv preprint hep-ph/0006269}, 2000.

\end{thebibliography}
\bibliographystyle{icml2021}

\end{document}